%% file: acml22_camera-ready_template.tex
\documentclass[wcp]{jmlr}

\usepackage{longtable}% for long tables

\usepackage{booktabs}
\usepackage{amsmath}
\usepackage{mathtools}
\usepackage{bm}
\usepackage{multicol}
\usepackage{multirow}
\usepackage[figuresright]{rotating}
\usepackage[]{graphicx}
\usepackage{graphicx}
\usepackage{float}
\usepackage{rotfloat}
\usepackage{caption}

\usepackage{algorithm}
\usepackage{algorithmic}

\DeclareMathOperator*{\argmin}{arg\,min}
\DeclareMathOperator*{\argmax}{\arg\max}

\newcommand{\bx}{\boldsymbol{x}}

\newcommand{\bg}{\boldsymbol{g}}

\newcommand{\bE}{\mathbb{E}}
\newcommand{\cX}{\mathcal{X}}
\newcommand{\cY}{\mathcal{Y}}

\newcommand{\cD}{\mathcal{D}}

\newcommand{\cG}{\mathcal{G}}
\newcommand{\cH}{\mathcal{H}}

\newcommand{\bz}{\boldsymbol{z}}

\makeatletter
\let\Ginclude@graphics\@org@Ginclude@graphics 
\makeatother

\jmlrvolume{189}
\jmlryear{2022}
\jmlrworkshop{ACML 2022}

\title[Multi-class Classification from Multiple Unlabeled Datasets]{Multi-class Classification from Multiple Unlabeled Datasets\\ with Partial Risk Regularization}

 \author{\Name{Yuting Tang} \Email{tang@ms.k.u-tokyo.ac.jp}\\
 \Name{Nan Lu} \Email{lu@ms.k.u-tokyo.ac.jp}\\
 \Name{Tianyi Zhang} \Email{zhang@ms.k.u-tokyo.ac.jp}\\
 \addr The University of Tokyo/RIKEN, Japan.
 \AND
 \Name{Masashi Sugiyama} \Email{sugi@k.u-tokyo.ac.jp}\\
 \addr RIKEN/The University of Tokyo, Japan.
 }

\editors{Emtiyaz Khan and Mehmet G\"{o}nen}

\begin{document}

\maketitle

\begin{abstract}
Recent years have witnessed a great success of \emph{supervised} deep learning, where predictive models were trained from a large amount of \emph{fully} labeled data. However, in practice, labeling such big data can be very costly and may not even be possible for privacy reasons.
Therefore, in this paper, we aim to learn an accurate classifier without any class labels. 
More specifically, we consider the case where multiple sets of \emph{unlabeled} data and only their \emph{class priors}, i.e., the proportions of each class, are available. 
Under this problem setup, we first derive an unbiased estimator of the classification risk that can be estimated from the given unlabeled sets and theoretically analyze the generalization error of the learned classifier.
We then find that the classifier obtained as such tends to cause overfitting as its empirical risks go negative during training.
To prevent overfitting, we further propose a partial risk regularization that maintains the partial risks with respect to unlabeled datasets and classes to certain levels.
Experiments demonstrate that our method effectively mitigates overfitting and outperforms state-of-the-art methods for learning from multiple unlabeled sets.
\end{abstract}

\input{sections/intro.tex}

\input{sections/preliminary.tex}
\input{sections/method.tex}
\input{sections/experiment.tex}
\input{sections/conclusion.tex}

\acks{NL and MS were supported by the Institute for AI and Beyond, UTokyo.
MS was also supported by JST AIP Acceleration Research Grant Number JPMJCR20U3, Japan.
TZ was supported by JST SPRING, Grant Number JPMJSP2108.}

\bibliography{acml22.bib}

\input{sections/appendix.tex}

% \section{Section Title}
% Main contents here.

% \subsection{Subsection Title}
% A figure in Fig.~\ref{fig:spiral}. Please use high quality graphics for your camera-ready submission -- if you can use a vector graphics format such as \texttt{.eps} or \texttt{.pdf}.
% \begin{figure}[htp]
% \begin{center}
% \includegraphics[width=0.5\textwidth]{spiral.eps}
% \caption{A spiral.}\label{fig:spiral}
% \end{center}
% \end{figure}

% An example of citation~\cite{DBLP:conf/acml/2009}.

% \acks{Acknowledgements should go at the end, before appendices and references.}

% %\bibliographystyle{plain}
% \bibliography{acml21}

% \appendix

% \section{First Appendix}\label{apd:first}

% This is the first appendix.

% \section{Second Appendix}\label{apd:second}

% This is the second appendix.

\end{document}

%% file: sections/intro.tex
\section{Introduction}
Supervised deep learning approaches are very successful but data-hungry \citep{goodfellow2016deep}.
As it is very expensive and time-consuming to collect full labels for big data, it is desirable for machine learning techniques to work with weaker forms of supervision \citep{10.1093/nsr/nwx106,sugiyama2022machine}. 

In this paper, we consider a challenging weakly supervised multi-class classification problem: instead of fully labeled data, we are given multiple sets of only unlabeled~(U) data for classifier training. We assume that each available unlabeled dataset varies only in class-prior probabilities, and these class priors are known \citep{lu2018minimal, lu2020mitigating}.
Our problem setting traces back to a classical problem of \emph{learning with label proportions} (LLP) \citep{quadrianto2009estimating}; however, there are two key differences between our work and previous LLP studies.
First, most of the existing LLP papers consider U sets sampled from the identical distribution \citep{liu2019learning,pmlr-v129-tsai20a}; while in our case, the U sets are generated from distributions with different class priors.
Second, for the solution, the majority of deep LLP methods are based on the \emph{empirical proportion risk minimization} (EPRM) \citep{yu2013proptosvm,yu2015learning} principle, which aims at predicting accurate label proportions for the U sets; while our method is based on the ordinary \emph{empirical risk minimization}~(ERM) \citep{Vapnik1998} principle, which is superior to EPRM as the consistency of learning can be guaranteed.

Such a framework of learning from multiple unlabeled datasets can find applications in various real-world scenarios.
For example, in politics, predicting the demographic information of users in social networks is critical to policy making \citep{culotta2015predicting};
however, labeling individual users according to their demographic information may violate data privacy, and thus the traditional supervised learning methods cannot be applied.
Fortunately, multiple unlabeled datasets can be obtained from different time points or geographic locations, and their class priors can be collected from pre-existing census data, which suffices for our learning framework.

Under this problem setup, the first challenge is to estimate the classification risk from only unlabeled datasets and their class priors, since the typical ERM approach used in supervised learning cannot be directly applied to this unlabeled setting.
A promising direction to solve the problem is risk rewriting \citep{sugiyama2022machine}, i.e., rewrite the classification risk into an equivalent form that can be estimated from the given data.
Existing risk rewriting results on learning from multiple U sets focused on binary classification \citep{lu2018minimal, lu2020mitigating}.
In this paper, we extend them and derive an unbiased risk estimator for multi-class classification from multiple U sets. In this way, ERM is enabled and statistical consistency \citep{mohri2018foundations} is guaranteed, i.e., the risk of the learned classifier converges to the risk of the optimal classifier in the model, as the amount of training data approaches infinity.

However, in practice, with only finite training data, we found that the risk rewriting method suffers from severe overfitting, which is another challenge we are facing.
For tackling overfitting, a common practice is to use regularization techniques to limit the flexibility of the model. 
Unfortunately, we also observed that these general-propose regularization techniques cannot mitigate this overfitting.
Based on another observation that the empirical training risks can go negative during training, we conjecture that the overfitting is due to the negative terms included in the rewritten risk function: when unbounded loss functions, e.g., the cross-entropy loss, are used, the empirical training risk can even diverge to negative infinity. Therefore, it is necessary to fix this negative risk issue.

Recently, \citet{ishida2020we} proposed the \emph{flooding} regularization to combat overfitting in supervised learning, which intentionally prevents further reduction of the empirical training risk when it reaches a reasonably small value called the \emph{flood level}.
We could naively apply the flooding regularization in our unlabeled setting and observed that overfitting can be mitigated to some extent. However, its effect was still on the same level as naive early stopping
\citep{morgan1989generalization}.

To further improve the performance, we propose a fine-grained \emph{partial risk regularization} inspired by the flooding method.
Specifically, we decompose the rewritten risk into partial risks, i.e., the risks regarding all data in each U set being in same class, and then maintain each partial risk to a certain level as flooding does.
We find that under some idealistic conditions, the optimal levels can be determined by the class priors, which serves as a guidance for selecting the hyper-parameters of these levels.
Experimental results show that our proposed partial risk regularization approach can successfully mitigate overfitting and outperform existing methods for learning from multiple U sets.

The rest of the paper is organized as follows. In Section~\ref{sec: problem setting}, we formulate the problem of learning from multiple U sets, and in Section~\ref{Unbiased Learning}, we propose an ERM method and analyze its estimation error. The partial risk regularizer is proposed in Section~\ref{sec: regularization},
and experimental results are discussed in Section~\ref{sec: experiment}. Finally, conclusions are given in Section~\ref{sec: concl}.

%% file: sections/preliminary.tex
\section{From Supervised Learning to Learning from Multiple Unlabeled Datasets}
\label{sec: problem setting}
To begin with, let us consider a multi-class classification task with the input feature space $\cX\subset \mathbb{R}^d$ and the output label space $\cY=\{ 1,2,\ldots,K\}=:[K]$, where $d$ is the input dimension and $K$ is the number of classes.
Let $\bx\in\cX$ and $y\in\cY$ be the input and output random variables following an underlying joint distribution with density $p(\bx,y)$, which can be decomposed using the class priors $\{\pi_k=\Pr(y=k)\}_{k=1}^K$ and the class-conditional densities $\{p_k(\bx)=p(\bx\mid y=k)\}_{k=1}^K$ as $p(\bx,y)=\sum_{k=1}^K \pi_k p_k(\bx)$.
The goal of multi-class classification is to learn a classifier $\bg:\cX\rightarrow \mathbb{R}^K$ that minimizes the following \emph{classification error}, also known as the \emph{risk}:
\begin{align}
    \label{eq: risk}
    R(\bg) \coloneqq \bE_{(\bx, y)\sim p(\bx, y)}\left[\ell(\bg(\bx), y)\right],
\end{align}
where $\bE$ denotes the expectation, and $\ell:\mathbb{R}^K\times\mathcal{Y}\to\mathbb{R}_{+}$ denotes a real-valued \emph{loss function} that measures the discrepancy between the true label $y$ and its prediction $\bg(\bx)$.
Typically, the predicted label is given by $\widehat{y}=\argmax_{k\in [K]}(\bg(\bx))_k$, where $(\bg(\bx))_k$ is the $k$-th element of $\bg(\bx)$.
Note that when \eqref{eq: risk} is used for evaluation, $\ell$ is often chosen as the \emph{zero-one loss} $\ell_\mathrm{01}(\bg(\bx), y)=I(\widehat{y}\neq y)$,
where $I$ is the indicator function;
when \eqref{eq: risk} is used for training, $\ell_\mathrm{01}$ is replaced by a \emph{surrogate loss} due to its difficulty for optimization, e.g., by the softmax cross-entropy loss.

As the density $p(\bx, y)$ in \eqref{eq: risk} remains unknown, in traditional supervised learning, we normally require a vast amount of fully labeled training data $\cD\coloneqq\left\{\left(\bx_i,  y_i\right)\right\}_{i=1}^{n}$ (where $n$ is the training sample size and is assumed to be a sufficiently large number) drawn independently from  $p(\bx,y)$ to learn an accurate classifier \citep{wahba1990spline,Vapnik1998,hastie2009elements,sugiyama2015introduction}.
With such supervised data, \emph{empirical risk minimization}~(ERM) \citep{Vapnik1998} is a common practice that approximates \eqref{eq: risk} by replacing the expectation with the average over the training data:
$\widehat{R}(\bg) = \frac{1}{n}\sum_{i=1}^{n}\left[\ell(\bg(\bx_i), y_i)\right]$; then minimizing the empirical risk $\widehat{R}(\bg)$ over a class of functions, also known as the \emph{model} $\mathcal{G}$: $\widehat{\bg}=\argmin_{\bg\in\mathcal{G}}\widehat{R}(\bg)$.

However, in practice, collecting massive fully labeled data may be difficult due to the high labeling cost.
In this paper, we consider a challenging setting of learning from multiple unlabeled datasets where \emph{no} explicit labels are given.
More specifically, assume that we have access to $M$ sets of unlabeled samples $\cD_\text{U}\coloneqq\bigcup_{m=1}^{M}\cX_m$, and each unlabeled set $\cX_m=\{\bx_{m,i}\}_{i=1}^{n_m}$ is a collection of $n_m$ data points drawn from a mixture of class-conditional densities $\{p_k=p(\bx\mid y=k)\}_{k=1}^K$:
\begin{align}
    \bx_{m,i}\overset{\text{i.i.d.}}\sim q_m(\bx)=\sum_{k=1}^{K} \theta_{m,k}p_{k},
\label{eqn:MCD framework}
\end{align}
where $\theta_{m,k}\geq0$ denotes the $k$-th class prior of the $m$-th unlabeled set. The class priors are assumed to form a full column rank matrix $\Theta\coloneqq (\theta_{m, k}) \in \mathbb{R}^{M \times K}$ with the constraint that $\sum_{k=1}^K\theta_{m,k}=1$.
Through the paper, we assume that the class priors of each unlabeled set are given, including training class priors $\Theta$ and test class priors $\{\pi_k\}_{k=1}^K$.%
\footnote{Later in Section~\ref{sec:robustness}, we will experimentally show that the use of approximate class-priors is sufficient to obtain an accurate classifier.}
Our goal is the same as standard supervised learning where we wish to learn a classifier $\bg$ that generalizes well with respect to $p(\bx,y)$, despite the fact that it is unobserved and we can only access unlabeled training sets $\{\cX_m\}_{m=1}^{M}$.

%% file: sections/method.tex
\section{Learning from Multiple Unlabeled Datasets}
\label{Unbiased Learning}
In this section, we derive an unbiased estimator of the classification risk in learning from multiple unlabeled datasets and analyze its estimation error.
% provide a theoretical analysis.

\subsection{Estimation of the Multi-Class Risk}
\label{subsec: rewritten risk}
Unbiased risk estimator based approaches have shown promising results in many weakly-supervised learning scenarios \citep{du2014analysis, du2015convex, ishida2017binary, lu2018minimal, van2017theory},
% However, the most related work to our paper, PU learning \citep{du2014analysis, du2015convex} and UU learning \citep{lu2018minimal}, 
most of which mainly focused on binary classification problems. 
In what follows, we extend the results of unbiased risk estimation in their papers to multi-class classification from multiple unlabeled datasets.

The key technique for deriving unbiased risk estimators is risk rewriting, which enables the risk to be calculated only from observable distributions. 
In the context of our setting, risk rewriting means to evaluate the risk using $\{q_m\}_{m=1}^M$ instead of unknown $\{p_k\}_{k=1}^K$. 
Formally, we extend the definition of rewritability of the risk introduced in \citet{lu2018minimal} to our setting and show that $R(\bg)$ is rewritable as follows.

\begin{definition}
We say that $R(\bg)$ is rewritable given $\{q_m\}_{m=1}^M$, if and only if there exist constants $\{w_{m, k}\}_{m\in[M],k\in[K]}$ such that for any model $\bg$ it holds that
\begin{align}
    \label{eq: rewritable}
R(\bg) = \sum_{m=1}^M \bE_{\bx \sim q_m}\left[\bar{\ell}_m(\bg(\bx))\right],
\end{align}
where $\bar{\ell}_m(\bz) = \sum_{k=1}^K w_{m, k} \ell(\bz, k)$.
\end{definition}

\begin{theorem}
\label{thm: risk rewriting}
Assume that $\Theta$ has full column rank. Then $R(\bg)$ is rewritable by letting
\begin{align}
    \label{eq: risk rewriting}
    W = (\Pi\Theta^\dagger)^\top,
\end{align}
where $W \coloneqq (w_{m, k}) \in \mathbb{R}^{M \times K}$, $\Pi \coloneqq \operatorname{diag} \{\pi_1, \pi_2, \ldots, \pi_K\}$, and $\dagger$ denotes the Moore–Penrose generalized inverse.
\end{theorem}

Based on Theorem~\ref{thm: risk rewriting}, we can immediately obtain a rewritten risk function as
\begin{align}
    \label{eq: rewritten risk}
    R_{\text{U}}(\bg) = \sum_{m=1}^M \sum_{k=1}^K w_{m, k} \bE_{q_m}\left[ \ell(\bg(\bx), k) \right],
\end{align}
and its unbiased risk estimator
\begin{align}
    \label{eq: unbiased risk estimator}
    \widehat{R}_{\text{U}}(\bg) = \sum_{m=1}^M \sum_{k=1}^K \frac{w_{m,k}}{n_m} \sum_{i=1}^{n_m} \ell(\bg(\bx_{m, i}), k),
\end{align}
where $w_{m, k}$'s are given by \eqref{eq: risk rewriting}. 
Then, any off-the-shelf optimization algorithms, such as stochastic gradient descent \citep{robbins1951stochastic} and Adam \citep{kingma2015adam}, can be applied for minimizing the empirical risk given by \eqref{eq: unbiased risk estimator}.

\subsection{Theoretical Analysis}
Here, we establish an estimation error bound for the unbiased risk estimator.
% The following Theorem~\ref{thm: bound} shows that as $n \rightarrow \infty$, the risk of the learned classifier
% % the empirical risk minimizer 
% $\widehat{\bg}_{\textup{U}}\coloneqq\arg \min_{\bg \in \cG} \widehat{R}_{\text{U}}(\bg)$ converges to 
% the risk of the optimal classifier in the given model class
% % the true risk minimizer 
% $\bg^{*} \coloneqq \arg \min_{\bg \in \cG} R(\bg)$.

\begin{theorem}
\label{thm: bound}
Let $\mathfrak{R}_{n_m}(\cH_{k})$ be the Rademacher complexity of $\cH_{k}$ over $q_m(\bx)$ \citep{mohri2018foundations}, i.e.,
% \begin{align*}
$\mathfrak{R}_{n_m}(\cH_{k}) = \bE_{q_m(\bx)} \bE_{\bm{\sigma}} \left[ \sup_{h \in \cH_k} \frac{1}{n_m} \sum_{i=1}^{n_m} \sigma_{i} h(\bx_i) \right]$,
% \end{align*}
where $\sigma_i$'s are independent uniform random variables taking values in $\{+1, -1\}$ and $\cH_{k} = \left\{ h: \bx \mapsto (\bg(\bx))_{k} | \bg \in \cG \right\}$.
Assume the loss function $\ell(\bg(\bx),y)$ is $L$-Lipschitz with respect to $\bg(\bx)$ for all $y \in \cY$ and upper-bounded by $C_{\ell}$, i.e., $C_{\ell} = \sup_{\bg \in \cG, \bx \in \cX, y \in \cY} \ell(\bg(\bx),y)$. 
$\widehat{\bg}_{\textup{U}}\coloneqq\arg \min_{\bg \in \cG} \widehat{R}_{\text{U}}(\bg)$ is the risk of the learned classifier and $\bg^{*} \coloneqq \arg \min_{\bg \in \cG} R(\bg)$ is the risk of the optimal classifier in the given model class.
Then, for any $\delta > 0$, with probability at least $1 - \delta$,
\begin{align}
    \label{eq: bound of unbiased risk estimator}
    R(\widehat{\bg}_{\textup{U}}) - R(\bg^{*}) \leq 4 \sqrt{2} LKC_{w} \sum_{m=1}^{M} \sum_{k=1}^{K} \mathfrak{R}_{n_m}(\cH_{k})
    + 2C_{w}KC_{\ell} \sqrt{\frac{{}M\ln{\frac{2}{\delta}}}{2N_{\min}}},
\end{align}
where $C_{w} = \max_{m\in[M], k\in[K]} |w_{m, k}|$ and $N_{\min} = \min_{m\in[M]} n_m$.
\end{theorem}

The proof of Theorem~\ref{thm: bound} is provided in Appendix~\ref{apd:bound}. Generally, $\mathfrak{R}_{n_m}(\cH_{k})$ can be bounded by $C_{\cG}/ \sqrt{N_{\mathrm{min}}}$ for a positive constant $C_{\cG}$ \citep{golowich2018size, NEURIPS2019_9308b0d6}. 
The Theorem~\ref{thm: bound} shows that as $N_{\mathrm{min}} \rightarrow \infty$, $\widehat{\bg}_{\textup{U}}$ converges to $\bg^{*}$.

\section{Partial Risk Regularization}
\label{sec: regularization}
% In this section, we first study the overfitting issue of the unbiased risk estimator of learning from unlabeled datasets and then propose a partial risk regularization approach for avoiding overfitting.
In this section, we empirically analyze the unbiased risk estimator in a finite-sample case, and then propose partial risk regularization for further boosting its performance.

\subsection{Overfitting of the Unbiased Risk Estimator}
\label{subsec: overfitting}
% As described in Section~\ref{Unbiased Learning}, we can use the risk rewriting technique to obtain an unbiased risk estimator. But in practice, we find that the unbiased risk estimation method may suffer severe overfitting when flexible models are used.
% training with finite samples. 
So far, we have obtained an unbiased risk estimator for learning from unlabeled datasets and derived an estimation error bound which guarantees the consistency of learning.
However, we find that in practice, with only finite training data, the unbiased method tends to overfit especially when flexible models are used (see Figure~\ref{fig:Reguarization} (a)). 

% Overfitting often occurs when the model trains for too long on training data or when the model is too complex. When the model fits too closely to the finite training samples, the model becomes overfitted and cannot generalize well to new data. 

For reducing overfitting, a common method is to use regularization techniques to limit the flexibility of the model. Popular choices of general-propose regularization techniques are dropout and weight decay \citep{goodfellow2016deep}. In Figure~\ref{fig:Reguarization},%
\footnote{The class priors $\bm{\Theta}$ were set to be an asymmetric diagonal dominated square matrix, see the detailed experimental setup in Section~\ref{Comparison with State-of-the-Art Methods}.}
(b) and (c) show the results of applying them on MNIST. For the dropout experiment, we fixed the weight decay parameter to be $1 \times 10^{-5}$ and changed the dropout probability from $0$ to $0.8$. For the weight decay experiment, we fixed the dropout probability to be $0$ and changed the weight decay parameter from $1 \times 10^{-5}$ to $1 \times 10^{1}$. 
The results show that general-propose regularization techniques cannot mitigate the overfitting of the unbiased risk estimation method in learning from multiple unlabeled datasets.

To further analyze this phenomenon, \footnote{This phenomenon is also observed in other weakly supervised learning settings using the unbiased risk estimator method, for example, in learning from positive and unlabeled data \citep{kiryo2017positive} and learning from two unlabeled datasets \citep{lu2020mitigating}.}
let us take a close look at the unbiased risk estimator (\ref{eq: unbiased risk estimator}). We find that the coefficients $w_{m, k}$'s given by Theorem~\ref{thm: risk rewriting} are not necessarily non-negative, some of which are exactly negative numbers.
These negative terms may lead the empirical training risk to go negative during training (see Figure~\ref{fig:Reguarization} (a)).
When overparameterized models, e.g., deep neural networks, and unbounded loss functions, e.g., the cross-entropy loss, are used, the empirical training risk may even diverge to negative infinity; the model can memorize all training data and become overly confident.
This may be the reasons why severe overfitting occurs here and why general-propose regularizations cannot solve it.

\begin{figure*}[t]
    \centering
    \subfigure[Negative training risk]{\includegraphics[height=3.5cm,width=5cm]{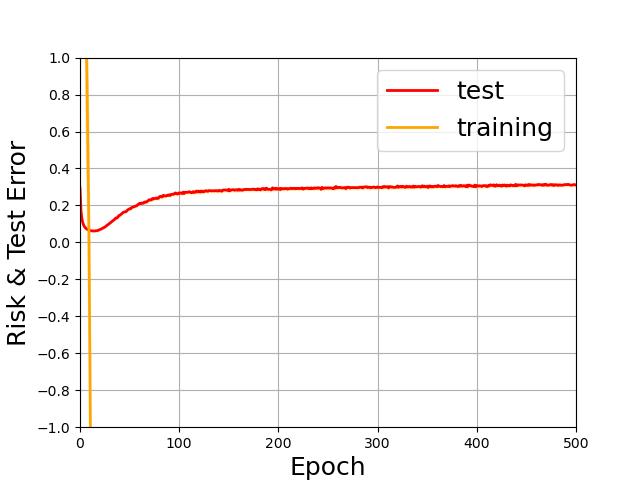}}
    \hfill
    \subfigure[Dropout]{\includegraphics[height=3.5cm,width=5cm]{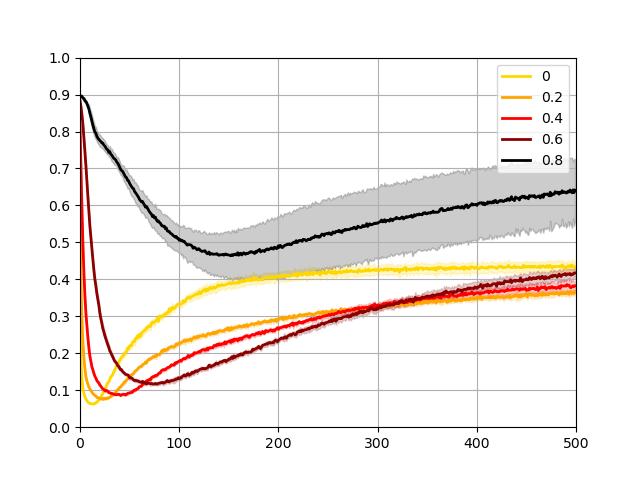}}
    \hfill
    \subfigure[Weight decay]{\includegraphics[height=3.5cm,width=5cm]{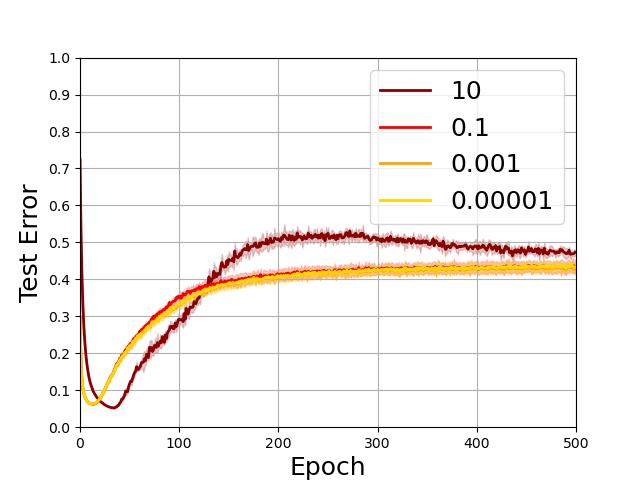}}
    \caption{The negative training risk and overfitting problem on test data are shown in (a). A 5-layer multi-layer perception (MLP) was trained by Adam \citep{kingma2015adam} using the cross-entropy loss as a surrogate loss function. And the test error was calculated by zero-one loss.
    The results of general-propose regularization techniques dropout and weight decay are presented in (b) and (c), respectively. Each method was trained for $500$ epochs.}
    \label{fig:Reguarization}
\end{figure*}

\subsection{Partial Risk Regularization}
Based on the analysis in Section~\ref{subsec: overfitting}, the issue that the empirical training risk goes negative and may even diverge should be fixed.
% should be lower bounded.
Flooding \citep{ishida2020we} is a regularization method dedicated to maintaining a fixed-level empirical training risk, which has been shown to mitigate overfitting effectively in the supervised setting.
This method seems promising to solve our problem, however, it is based on supervised learning and needs full labels to calculate the risk.
Thanks to the unbiased risk estimator derived in Section~\ref{subsec: rewritten risk}, we can straightforwardly apply the flooding method to it as follows:
\begin{align}
    \label{eq: flooding}
    % \widehat{R}_{\text{U-flood}}(\bg) = |\widehat{R}_{\text{U}}(\bg)-b|+b,
    R_{\text{U-flood}}(\bg) = |R_{\text{U}}(\bg)-b|+b,
\end{align}
where $b \ge 0$ is a hyperparameter called the flood level.
The experimental results in Section~\ref{sec: experiment} show this simple method relieves overfitting to some extent,
but only performs on the same level of early stopping.

To further boost performance, we decompose the total risk
% $\widehat{R}_{\text{U}}(\bg)$ 
$R_{\text{U}}(\bg)$ into partial risks:
\begin{align}
    \label{eq: class partial risk}
    R_{q_m,k}(\bg) \coloneqq \bE_{q_m}\left[\ell(\bg(\bx), k) \right],
    % \frac{1}{n_m} \sum_{i=1}^{n_m} \ell(\bg(\bx_{m, i}), k),
\end{align}
that is, the risk of considering all data in one unlabeled dataset to belong to the same class.
Then we propose a fine-grained \emph{partial risk regularization}~(PRR) approach that maintains each partial risk in the rewritten risk function~\eqref{eq: rewritten risk} to a certain level:
\begin{equation}
    R_{\text{U-PRR}}(\bg) = \alpha R_{\text{U}}(\bg)+ (1-\alpha) R_{\text{reg}}(\bg),
    \label{eqn: objective}
\end{equation}
where
\begin{equation}
    R_{\text{reg}}(\bg) = \sum_{m=1}^M \sum_{k=1}^K \lambda_{m, k} \left | R_{q_{m},k}(\bg) - b_{m,k} \right |,
    \label{eqn: regularization term}
\end{equation}
$0 \leq \alpha \leq 1$ and $\lambda_{m, k}$'s are trade-off parameters, and $b_{m,k}$'s are the flood levels for the corresponding partial risks.
To select the flood levels, we provide the following lemma which studies the relationship between them and the Bayes optimal classifier $\bg^\star$ where $R(\bg^\star) = \inf_{\bg} R(\bg)$.

\begin{lemma}
Let $R_{p,k}(\bg)=\bE_{p}\left[\ell(\bg(\bx), k)\right]$ where $p$ is any probability density function, $k=1,\ldots,K$, and $m=1,\ldots,M$.
We have
$R_{q_m,k}(\bg^\star)=\sum_{i=1}^{K} \theta_{m,i} R_{p_i,k}(\bg^\star)$.
When all classes are separable, we have
$R^{01}_{q_m,k}(\bg^\star) = 1 - \theta_{m,k}$,
where $R^{01}_{q_m,k}(\bg)$ denotes $R_{q_m,k}(\bg)$ calculated by the zero-one loss.
\label{lemma: float condition}
\end{lemma}

Guided by Lemma~\ref{lemma: float condition}, in order to push $\bg$ towards the optimal classifier $\bg^\star$, we propose to choose $b_{m,k} = 1 - \theta_{m,k}$ as the flood level of $R^{01}_{q_m,k}(\bg)$, and replace $R_{q_m,k}(\bg)$ by $R^{01}_{q_m,k}(\bg)$ in equation~\eqref{eqn: regularization term}. 
Therefore, the empirical learning objective of PRR approach is
\begin{equation}
    \widehat{R}_{\text{U-PRR}}(\bg) = \alpha \widehat{R}_{\text{U}}(\bg)+ (1-\alpha) \widehat{R}_{\text{reg}}(\bg),
    \label{eqn: empirical objective}
\end{equation}
where 
\begin{equation}
    \widehat{R}_{\text{reg}}(\bg) = \sum_{m=1}^M \sum_{k=1}^K \lambda_{m, k} \left | \widehat{R}^{01}_{q_{m},k}(\bg) - (1-\theta_{m,k}) \right |,
    \label{eqn: empirical regularization term}
\end{equation}
and
\begin{equation}
\widehat{R}^{01}_{q_{m},k}(\bg) = \frac{1}{n_m} \sum_{i=1}^{n_m} \ell_{01}(\bg(\bx_{m, i}), k).
\end{equation}

In the implementation (see Algorithm~\ref{alg:example}), we use the zero-one loss for determining the sign of the term in the absolute function and use a surrogate loss in calculating the gradients due to the difficulty of optimizing the zero-one loss. For dealing with absolute functions, when the term in it is greater than $0$, we perform gradient descent as usual; when the term is less than $0$, we perform gradient ascent instead and add a hyperparameter $s_\mathrm{GA}$ in the gradient ascent process. 

In addition, we have to determine hyperparameters $\lambda_{m,k}$ in the regularization part.
Here, we set them as $\lambda_{m,k} = \left |w_{m,k}  \right |$ since choosing such a large number of hyperparameters separately is not straightforward. 
This choice means that the term that receives more attention in the unbiased risk estimator also receives more attention in partial risk regularization. 

\begin{algorithm}[t]
\small
\textbf{Input}: $M$ sets of U training data $\cD_\text{U} = \bigcup_{m=1}^{M}\cX_m$ with known class prior $\Theta$\\
\textbf{Output}: model parameter $\gamma$ for classifier $\bg$
\begin{algorithmic}[1]
    \STATE{Initialize $\gamma$}\\
    \STATE{Let $\mathcal{A}$ be an SGD-like optimizer working on $\gamma$}\\
    \FOR{$t=0$; $t<\textrm{number of epochs}$; $t++$ }
    \STATE{Shuffle $\cD_\mathrm{U}$}\\
    \FOR{$i=0$; $i<\textrm{number of mini-batches}$; $i++$ }
        \STATE{Let $\overline{\cD}_\text{U} = \bigcup_{m=1}^{M}\overline{\cX}_m$ be the current mini-batch}\\
        \STATE{Compute $\widehat{R}_{\text{U}}(\bg)$ followed Eq.~(\ref{eq: unbiased risk estimator})}\\
        \STATE{Compute $\widehat{R}_{\text{reg}}(\bg)$:}\\
        \IF{$\bE_{\overline{\cX}_m}\left[\ell_{01}(\bg(\bx), k)\right] \ge (1 - \theta_{m, k})$}
            \STATE{$\left | \bE_{\overline{\cX}_m}\left[\ell(\bg(\bx), k)\right] - (1 - \theta_{m, k})\right | \gets \bE_{\overline{\cX}_m}\left[\ell(\bg(\bx), k)\right] - (1 - \theta_{m, k})$}\\
        \ELSE
            \STATE{$\left | \bE_{\overline{\cX}_m}\left[\ell(\bg(\bx), k)\right] - (1 - \theta_{m, k})\right | \gets -s_\mathrm{GA}(\bE_{\overline{\cX}_m}\left[\ell(\bg(\bx), k)\right] - (1 - \theta_{m, k}))$}\\
        \ENDIF
        \STATE{Compute $\widehat{R}_{\text{U-PRR}}(\bg)$} followed Eq.~(\ref{eqn: empirical objective})\\
        \STATE{Compute gradient $\bigtriangledown_\gamma\widehat{R}_{\text{U-PRR}}(\bg)$}\\
        \STATE{Update $\gamma$ by $\mathcal{A}$}
    \ENDFOR
    \ENDFOR
\end{algorithmic}
\caption{Partial risk regularization approach for learning from multiple U datasets}
\label{alg:example}
\end{algorithm}

%% file: sections/experiment.tex
\section{Experiments}
\label{sec: experiment}
In this section, we verify the effectiveness of the proposed method\footnote{Our implementation is available at \href{https://github.com/Tang-Yuting/U-PRR}{https://github.com/Tang-Yuting/U-PRR}.} on various datasets. We also show the robustness against noisy class priors of the proposed method.

\subsection{Experimental Setup}
We describe the details of the experimental setup as follows.

\paragraph{Datasets}
We used widely adopted benchmarks: Pendigits \citep{alimoglu1996methods}, USPS \citep{USPS}, MNIST \citep{lecun1998gradient}, Fashion-MNIST \citep{xiao2017fashion}, Kuzushiji-MNIST \citep{clanuwat2018deep} and, CIFAR-10 \citep{krizhevsky2009learning}. Table \ref{tab:dataset} briefly summaries the benchmark datasets.

    \begin{table*}[t]
    % \small
    \centering
    \caption{Specification of datasets and corresponding models.}
    \scalebox{0.7}{\begin{tabular}{c|ccccc}
    \toprule
    Dataset& Feature& Class& Training Data& Test Data& Model\\
    \midrule 
    Pendigits& 16& 10& 7,494& 3,498& FC with ReLU (depth 3)\\
    USPS& 16*16& 10& 7,291& 2,007& FC with ReLU (depth 3)\\
    MNIST& 28*28& 10& 60,000& 10,000& FC with ReLU (depth 5)\\
    Fashion-MNIST& 28*28& 10& 60,000& 10,000& FC with ReLU (depth 5)\\
    Kuzushiji-MNIST& 28*28& 10& 60,000& 10,000& FC with ReLU (depth 5)\\
    CIFAR-10& 32*32*3& 10& 50,000& 10,000& ResNet (depth 20)\\
    \bottomrule
    \end{tabular}}
    \label{tab:dataset}
    \end{table*}
The $m$-th U set contains $n_{m}$ samples where its proportions $\theta_{m,k}$ were randomly chosen from class $k$ for $k=1,\ldots,K$. The size of each unlabeled dataset $n_{m}$ is set to $1/M$ of the total number of training datasets in all these settings.

\paragraph{Baselines}
To better analyze the performance of the proposed unbiased risk estimator~\eqref{eq: unbiased risk estimator} (\emph{Unbiased}) and partial risk regularization approach~\eqref{eqn: empirical objective} (\emph{U-PRR}), we compared them with following baseline methods:
    \begin{itemize}
        \item Biased proportion (\emph{Biased}): We consider the label of the largest category in each unlabeled dataset as the label of all samples in it, and perform supervised learning from such pseudo-label samples.

        \item Proportion loss (\emph{Prop}/\emph{Prop-CR}): \emph{Prop} uses class priors of each unlabeled dataset as weak supervision and aims to minimize the difference between true class priors and the predicted class priors \citep{yu2015learning}. \emph{Prop-CR} builds upon the proportion loss baseline by adding a consistency regularization term \citep{pmlr-v129-tsai20a}.
        
        % \item Proportion loss (\emph{Prop}) \citep{yu2015learning}: It uses class priors of each unlabeled dataset as weak supervision and aims to minimize the difference between true class priors and the predicted class priors.

        % \item Proportion loss with consistency regularization (\emph{Prop-CR}) \citep{pmlr-v129-tsai20a}: It builds upon the proportion loss baseline by adding a consistency regularization term.
        
        \item Corrected unbiased risk estimator (\emph{U-correct}): We
        apply a corrected risk estimation method \citep{lu2020mitigating}, which aims to correct partial risks corresponding to each class to be non-negative, to the unbiased risk estimator.
        The learning objective is $\widehat{R}_{\text{U-correct}}(\bg) = \sum_{k=1}^K \left|  \sum_{m=1}^M w_{m, k} \frac{1}{n_m}\sum_{i=1}^{n_m}\ell(\bg(\bx_{m, i}), k)\right|$.
        
        \item Unbiased risk estimator with early stopping (\emph{U-stop}): We apply the early stopping on the proposed unbiased risk estimator, and stop the training when the empirical risk goes negative since \citet{lu2020mitigating} showed a high co-occurrence between overfitting and negative empirical risk.
        
        \item Unbiased risk estimator with flooding (\emph{U-flood}) \citep{ishida2020we}: We directly apply the flooding method to the unbiased risk estimator. The learning objective is $\widehat{R}(\bg) = |\widehat{R}_{\text{U}}(\bg)-b|+b$. We choose $b$ from $\{0, 0.05, 0.1\}$ and report the best test error.
    \end{itemize}

\paragraph{Training Details}
The models are described in Table \ref{tab:dataset}, where FC refers to the fully connected neural network and ResNet refers to the residual network \citep{he2016deep}. As a common practice, Adam \citep{kingma2015adam} with the cross-entropy loss was used for optimization. We ran each experiment five times, and we trained the model for $500$ epochs on all datasets 
except that CIFAR-10 only used $200$ epochs since $200$ epochs were sufficient for convergence.
The classification error rate was used for evaluating the test performance.
The learning rates were chosen from $\{5\times10^{-5}, 1\times10^{-4}, 2\times10^{-4}, 5\times10^{-4}, 1\times10^{-3} \}$ and the batch sizes were chosen from $\{\frac{n}{500}, \frac{n}{200}, \frac{n}{100}, \frac{n}{50}, \frac{n}{20}, \frac{n}{10}\}$, where $n$ is the size of the training dataset.
In the PRR method, $\alpha$'s were chosen from $\{ 0.1, 0.3, 0.5, 0.7, 0.9 \}$ and $s_\mathrm{GA}$'s were chosen from $\{ 0.1, 0.2, 0.5, 1, 2, 5, 10 \}$.
Note that for fair comparison, we used the same model and the same hyperparameters for the implementation of all methods in each benchmark. The implementation in our experiments was based on PyTorch\footnote{See \href{https://pytorch.org}{https://pytorch.org}.}, and the experiments were conducted on an NVIDIA Tesla P100 GPU.

\subsection{Comparison with Baseline Methods}
\label{Comparison with State-of-the-Art Methods}
Here, we report the final test error (Err) and the test error drop ($\Delta_{E}$), which is the difference between the smallest test error among all training epochs and the test error at the end of training, for the baseline methods and proposed methods.
We designed different $\Theta$ for evaluating the proposed methods in various scenarios:

\begin{itemize}
    \item Symmetric square class-prior matrix. It is defined as
    \begin{equation*}
    \begin{aligned}
        \theta_{m,k}
        = \begin{cases}
    a + b & m = k, \\ 
    b & m \ne k, \\
    \end{cases}
        \label{eqn:label proportions of 10 bags}
    \end{aligned}
    \end{equation*}
    where $a > 0$ and $b > 0$ are constants satisfying $a + Kb = 1$. 
    We tested our proposed methods under two different training class-prior settings: $(a, b)$ was chosen as $(0.5, 0.05)$ and $(0.1, 0.09)$. The experimental results are reported in Table~\ref{tab:experiments on symmetric class-prior matrix}. 
    
    \begin{table}[t]
    \centering
    \caption{Experimental results with a symmetric class-prior matrix. Means (standard deviations) of the classification error (Err) and the drop ($\Delta_{E}$) over five trials in percentage. The best and comparable methods based on the paired t-test at the significance level $5\%$ are highlighted in boldface.}
    \scalebox{0.55}{\begin{tabular}{c|c|cc|cc|cc|cc|cc|cc|cc|cc}
    \toprule
    \multirow{2}*{Dataset}& \multirow{2}*{a,b}&
    \multicolumn{2}{c|}{Biased}&
    \multicolumn{2}{c|}{Prop}&
    \multicolumn{2}{c|}{Prop-CR}&
    \multicolumn{2}{c|}{Unbiased}&
    \multicolumn{2}{c|}{U-correct}&
    \multicolumn{2}{c|}{U-stop}&
    \multicolumn{2}{c|}{U-flood}&
    \multicolumn{2}{c}{U-PRR}\\
    & & Err& $\Delta_{E}$& Err& $\Delta_{E}$& Err& $\Delta_{E}$& Err& $\Delta_{E}$& Err& $\Delta_{E}$& Err& $\Delta_{E}$& Err& $\Delta_{E}$& Err& $\Delta_{E}$\\
    \midrule  %添加表格中横线
    \multirow{4}*{Pendigits}& \multirow{2}*{0.5, 0.05}& 11.54& 7.42& \textbf{5.80}& 0.23& 5.85& 0.31& 17.99 & 13.09&\textbf{5.58}& 1.25& 7.80& 0& \textbf{5.29}& 1.18& \textbf{4.61}& 0.26\\
    & & (1.19)& (1.05)& \textbf{(1.08)}& (0.19)& (1.01)& (0.22)& (1.88)& (1.67)& \textbf{(0.48)}& (0.38)& (1.13)& (0)& \textbf{(0.99)}& (0.96)& \textbf{(0.77)}& (0.30)\\
    & \multirow{2}*{0.1, 0.09}& 61.30& 41.27& 37.97& 22.44& 36.20& 20.90& 47.16& 19.45& 36.25& 16.28& 20.12& 0& 29.44& 11.12& \textbf{12.64}& 0.59\\
    & & (1.50)& (2.16)& (1.30)& (0.54)& (1.98)& (1.62)& (2.24)& (2.77)& (2.92)& (1.68)& (1.75)& (0)& (3.34)& (3.96)& \textbf{(2.48)}& (0.39)\\
    \midrule
    \multirow{4}*{USPS}& \multirow{2}*{0.5, 0.05}& 23.23& 14.73& 7.92& 0.62& 6.51& 0.33& 26.98& 18.64& 11.14& 2.11& 9.39& 0& 11.41& 2.53& \textbf{6.01}& 0.30\\
    & & (1.30)& (1.41)& (0.15)& (0.27)& (0.24)& (0.20)& (2.51)& (2.35)& (0.25)& (0.58)& (0.51)& (0)& (0.65)& (0.68)& \textbf{(0.25)}& (0.11)\\
    & \multirow{2}*{0.1, 0.09}& 70.42& 37.39& 66.97& 49.25& 56.52& 44.26& 75.86& 57.28& 69.97& 35.17& 22.32& 0& 56.09& 31.83& \textbf{11.15}& 0.64\\
    & & (0.60)& (1.95)& (3.14)& (3.28)& (3.00)& (3.42)& (1.50)& (1.44)& (2.08)& (3.94)& (2.09)& (0)& (3.42)& (3.16)& \textbf{(1.20)}& (0.12)\\
    \midrule
    \multirow{4}*{MNIST}& \multirow{2}*{0.5, 0.05}& 25.19& 19.02& 5.56& 0.09& 3.86& 0.08& 31.14& 24.93& 5.59& 0.24& 6.67& 0& 6.34& 0.42& \textbf{2.99}& 0.19\\
    & & (0.70)& (0.70)& (0.11)& (0.06)& (0.11)& (0.06)& (0.49)& (0.65)& (0.29)& (0.20)& (0.25)& (0)& (0.33)& (0.22)& \textbf{(0.10)}& (0.07)\\
    & \multirow{2}*{0.1, 0.09}& 78.19& 35.31& 66.21& 50.56& 43.35& 40.37& 82.51& 63.83& 53.16& 24.39& 20.95& 0& 27.72& 6.34& \textbf{5.84}& 0.26\\
    & & (0.37)& (0.91)& (1.60)& (1.38)& (13.76)& (13.77)& (0.52)& (0.52)& (3.41)& (3.31)& (1.31)& (0)& (1.74)& (2.02)& \textbf{(0.19)}& (0.09)\\
    \midrule
    & \multirow{2}*{0.5, 0.05}& 32.57& 17.79& 14.75& 0.29& 15.25& 0.26& 36.03& 21.36& 15.14& 0.75& 14.88& 0.03& 15.95& 0.92& \textbf{13.09}& 0.50\\
    Fashion-& & (0.81)& (0.85)& (0.36)& (0.33)& (0.19)& (0.05)& (0.73)& (0.79)& (0.36)& (0.33)& (0.31)& (0.04)& (0.56)& (0.64)& \textbf{(0.17)}& (0.12)\\
    MNIST& \multirow{2}*{0.1, 0.09}& 79.28& 41.39& 56.34& 35.16& 34.81& 9.76& 82.85& 60.68& 39.45& 14.07& 23.39& 0& 26.43& 2.74& \textbf{21.91}& 2.78\\
    & & (0.54)& (0.88)& (1.15)& (1.37)& (1.60)& (1.94)& (0.64)& (0.62)& (2.95)& (2.87)& (0.72)& (0)& (1.67)& (1.25)& \textbf{(0.90)}& (0.76)\\
    \midrule
    & \multirow{2}*{0.5, 0.05}& 39.35& 15.84& 20.07& 0.41& 17.30& 0.17& 43.39& 21.38& 20.05& 0.96& 22.94& 0& 21.04& 0.99& \textbf{14.33}& 0.88\\
    Kuzushiji-& & (0.49)& (0.55)& (0.53)& (0.22)& (0.91)& (0.17)& (0.70)& (0.93)& (0.40)& (0.33)& (0.60)& (0)& (1.04)& (0.74)& \textbf{(0.14)}& (0.13)\\
    MNIST& \multirow{2}*{0.1, 0.09}& 81.26& 17.19& 73.10& 28.41& 39.33& 11.72& 84.57& 39.26& 61.77& 8.64& 46.57& 0& 47.23& 5.28& \textbf{31.17}& 3.09\\
    & & (0.58)& (1.15)& (2.54)& (2.15)& (3.30)& (3.28)& (0.46)& (0.51)& (2.33)& (1.48)& (1.13)& (0)& (2.38)& (2.92)& \textbf{(0.70)}& (0.89)\\
    \midrule
    \multirow{4}*{CIFAR-10}& \multirow{2}*{0.5, 0.05}& 51.41& 7.65& 65.24& 0.20& 67.76& 0.10& 57.44& 9.39& 57.72& 8.40& 51.94& 3.14& 56.49& 8.06& \textbf{44.75}& 1.37\\
    & & (3.10)& (2.95)& (1.87)& (0.19)& (1.68)& (0.21)& (3.20)& (2.18)& (5.93)& (6.77)& (2.64)& (2.39)& (4.19)& (3.16)& \textbf{(1.85)}& (0.86)\\
    & \multirow{2}*{0.1, 0.09}& 84.67& 16.66& 72.73& 0.41& 74.22& 0.17& 86.82& 19.41& 76.46& 8.96& \textbf{69.99}& 1.99& 71.08& 3.71& \textbf{69.24}& 3.14\\
    & & (0.25)& (1.68)& (2.83)& (0.37)& (1.32)& (0.23)& (1.71)& (1.39)& (2.19)& (1.17)& \textbf{(1.25)}& (1.76)& (2.15)& (2.15)& \textbf{(1.55)}& (1.46)\\
    \bottomrule
    \end{tabular}}
    \label{tab:experiments on symmetric class-prior matrix}
    \end{table}

    \item Asymmetric diagonal-dominated square class-prior matrix. In this setting, the values on the diagonals of the matrix $\Theta$ are larger than the other values, and they can be different from each other. We randomly and uniformly generated the class priors $\theta_{m,k}$ from range $[0, 1/M]$ when $m \neq k$, and set $\theta_{k,k} = 1 - \sum_{m\neq k}\theta_{m,k}$. The experimental results are shown in Table~\ref{tab:experiments on diagonal dominated class-prior matrix}.

    \begin{table}[ht]
    \centering
    \caption{Experimental results on diagonal-dominated square class-prior matrix. Means (standard deviations) of the classification error (Err) and the drop ($\Delta_{E}$) over five trials in percentage. The best and comparable methods based on the paired t-test at the significance level $5\%$ are highlighted in boldface.}
    \scalebox{0.55}{\begin{tabular}{c|cc|cc|cc|cc|cc|cc|cc|cc}
    \toprule
    \multirow{2}*{Dataset}&
    \multicolumn{2}{c|}{Biased}&
    \multicolumn{2}{c|}{Prop}&
    \multicolumn{2}{c|}{Prop-CR}&
    \multicolumn{2}{c|}{Unbiased}&
    \multicolumn{2}{c|}{U-correct}&
    \multicolumn{2}{c|}{U-stop}&
    \multicolumn{2}{c|}{U-flood}&
    \multicolumn{2}{c}{U-PRR}\\
    & Err& $\Delta_{E}$& Err& $\Delta_{E}$& Err& $\Delta_{E}$& Err& $\Delta_{E}$& Err& $\Delta_{E}$& Err& $\Delta_{E}$& Err& $\Delta_{E}$ & Err& $\Delta_{E}$\\
    \midrule  %添加表格中横线
    \multirow{2}*{Pendigits}& 15.19& 9.56& \textbf{6.75}& 0.35& \textbf{6.58}& 0.45& 27.06& 20.17& 9.75& 1.88& 10.29& 0.01& 7.92& 2.04& \textbf{5.57}& 0.86\\
    & (5.35)& (3.94)& \textbf{(0.59)}& (0.25)& \textbf{(0.48)}& (0.26)& (7.82)& (6.06)& (3.87)& (1.51)& (2.89)& (0.01)& (1.75)& (0.79)& \textbf{(1.78)}& (0.41)\\
    \midrule
    \multirow{2}*{USPS}& 27.85& 17.18& \textbf{10.90}& 2.18& \textbf{8.99}& 1.13& 46.06& 35.98& 26.14& 11.73& 11.27& 0& 14.70& 3.71& \textbf{7.31}& 0.48\\
    & (8.12)& (5.68)& \textbf{(3.45)}& (2.41)& \textbf{(2.64)}& (1.60)& (14.14)& (12.47)& (7.42)& (5.41)& (1.86)& (0)& (4.25)& (2.03)& \textbf{(1.02)}& (0.29)\\
    \midrule
    \multirow{2}*{MNIST}& 28.42& 21.31& 5.70& 0.11& 4.02& 0.05& 43.36& 37.12& 5.96& 0.37& 6.62& 0& 6.10& 0.24& \textbf{3.65}& 0.38\\
    & (0.60)& (0.68)& (0.17)& (0.09)& (0.08)& (0.05)& (2.01)& (1.90)& (0.44)& (0.27)& (0.24)& (0)& (0.30)& (0.21)& \textbf{(0.08)}& (0.08)\\
    \midrule
    Fashion-& 35.27& 19.93& \textbf{14.75}& 0.21& 14.96& 0.09& 45.09& 30.08& 15.35& 0.76& 15.10& 0.09& 15.60& 0.72& \textbf{14.25}& 1.21\\
    MNIST& (0.23)& (0.36)& \textbf{(0.20)}& (0.21)& (0.12)& (0.05)& (0.89)& (0.99)& (0.41)& (0.22)& (0.18)& (0.08)& (0.11)& (0.14)& \textbf{(0.19)}& (0.29)\\
    \midrule
    Kuzushiji-& 38.49& 14.88& 20.77& 0.68& \textbf{17.38}& 0.24& 52.17& 29.19& 20.18& 1.07& 23.89& 0& 20.50& 0.59& \textbf{17.15}& 1.28\\
    MNIST& (0.63)& (0.49)& (0.60)& (0.44)& \textbf{(0.64)}& (0.15)& (0.97)& (0.86)& (0.33)& (0.40)& (1.24)& (0)& (0.47)& (0.52)& \textbf{(0.71)}& (0.45)\\
    \midrule
    \multirow{2}*{CIFAR-10}& 50.05& 6.38& 64.96& 0.14& 68.67& 0.34& 61.45& 10.28& 56.04& 6.07& 53.32& 3.13& 57.49& 7.33& \textbf{45.02}& 1.92\\
    & (1.64)& (1.13)& (2.09)& (0.15)& (2.44)& (0.07)& (7.26)& (7.21)& (2.42)& (1.96)& (0.95)& (0.70)& (1.91)& (1.78)& \textbf{(1.85)}& (0.87)\\
    \bottomrule
    \end{tabular}}
    \label{tab:experiments on diagonal dominated class-prior matrix}
    \end{table}
    
    \item Non-square class-prior matrix. We also conduct experiments when the number of unlabeled datasets $M$ is more than the number of classes $K$. Here, we set $M = 2K$. The class-prior matrix $\Theta$ is a concatenation of two asymmetric diagonal-dominated square matrices. Compared with the previous experiments, this setting is more in line with real-world situations since it is rare that the number of unlabeled datasets is exactly equal to the number of classes.
    The experimental results are reported in Table~\ref{tab:experiments on multiple unlabeled datasets}.

    \begin{table}[ht]
    \centering
    \caption{Experimental results on non-square class-prior matrix. Means (standard deviations) of the classification error (Err) and the drop ($\Delta_{E}$) over five trials in percentage. The best and comparable methods based on the paired t-test at the significance level $5\%$ are highlighted in boldface.}
    \scalebox{0.55}{\begin{tabular}{c|cc|cc|cc|cc|cc|cc|cc|cc}
    \toprule
    \multirow{2}*{Dataset}&
    \multicolumn{2}{c|}{Biased}&
    \multicolumn{2}{c|}{Prop}&
    \multicolumn{2}{c|}{Prop-CR}&
    \multicolumn{2}{c|}{Unbiased}&
    \multicolumn{2}{c|}{U-correct}&
    \multicolumn{2}{c|}{U-stop}&
    \multicolumn{2}{c|}{U-flood}&
    \multicolumn{2}{c}{U-PRR}\\
    & Err& $\Delta_{E}$& Err& $\Delta_{E}$& Err& $\Delta_{E}$& Err& $\Delta_{E}$& Err& $\Delta_{E}$& Err& $\Delta_{E}$& Err& $\Delta_{E}$& Err& $\Delta_{E}$\\
    \midrule  %添加表格中横线
    \multirow{2}*{Pendigits}& 14.03& 9.65& \textbf{5.04}& 0.46& \textbf{5.05}& 0.46& 23.37& 17.68& 6.51& 1.46& 9.51& 0& 6.90& 2.04& \textbf{4.75}& 0.68\\
    & (3.69)& (2.79)& \textbf{(0.85)}& (0.18)& \textbf{(1.00)}& (0.25)& (6.59)& (5.58)& (1.27)& (0.63)& (2.44)& (0)& (1.81)& (1.49)& \textbf{(1.22)}& (0.77)\\
    \midrule
    \multirow{2}*{USPS}& 30.05& 19.59& 13.38& 4.47& \textbf{8.45}& 1.27& 45.04& 34.82& 21.02& 6.95& 10.89& 0.03& 13.38& 2.59& \textbf{8.00}& 0.82\\
    & (4.94)& (3.71)& (0.59)& (0.50)& \textbf{(0.29)}& (0.26)& (2.00)& (2.31)& (1.15)& (1.51)& (0.65)& (0.04)& (0.74)& (0.88)& \textbf{(0.30)}& (0.22)\\
    \midrule
    \multirow{2}*{MNIST}& 36.15& 26.56& 7.00& 0.69& 5.60& 0.29& 51.55& 43.82& 9.80& 1.19& 8.03& 0& 6.31& 0.60& \textbf{3.32}& 0.22\\
    & (0.43)& (0.41)& (0.40)& (0.39)& (0.14)& (0.08)& (1.45)& (1.57)& (0.68)& (0.78)& (0.30)& (0)& (0.56)& (0.44)& \textbf{(0.13)}& (0.09)\\
    \midrule
    Fashion-& 37.93& 20.60& \textbf{16.01}& 0.55& \textbf{15.27}& 0.23& 54.78& 38.31& 19.51& 2.44& 16.64& 0.02& 17.21& 0.95& \textbf{16.43}& 2.01\\
    MNIST& (0.90)& (0.78)& \textbf{(0.21)}& (0.28)& \textbf{(0.39)}& (0.15)& (2.69)& (2.37)& (0.47)& (0.53)& (0.25)& (0.04)& (0.85)& (0.88)&  \textbf{(0.86)}& (0.97)\\
    \midrule
    Kuzushiji-& 46.61& 19.04& 25.87& 0.49& \textbf{20.28}& 0.84& 57.52& 30.68& 27.39& 2.21& 27.53& 0& 23.33& 1.12& 26.40& 2.95\\
    MNIST& (0.98)& (0.90)& (0.68)& (0.81)& \textbf{(0.58)}& (0.53)& (1.49)& (1.78)& (0.44)& (0.49)& (0.74)& (0)& (0.91)& (0.70)& (0.74)& (1.15)\\
    \midrule
    \multirow{2}*{CIFAR-10}& 55.84& 7.47& 63.45& 0.21& 67.58& 0.10& 71.46& 15.30& 63.98& 9.82& 59.53& 7.05& 59.90& 7.32& \textbf{47.18}& 1.95\\
    & (2.04)& (2.04)& (2.02)& (0.17)& (1.74)& (0.17)& (11.81)& (11.97)& (2.55)& (2.60)& (6.53)& (5.84)& (1.40)& (0.45)& \textbf{(2.20)}& (0.86)\\
    % \multirow{2}*{Cifar-10}& & & 79.40& 29.22& 52.22& 4.38& 64.73& 14.63& 48.12& 8.10& 48.85& 8.10\\
    % & & & (6.66)& (6.28)& (1.90)& (1.66)& (7.34)& (7.38)& (0.72)& (0.32)& (0.58)& (0.55)\\
    \bottomrule
    \end{tabular}}
    \label{tab:experiments on multiple unlabeled datasets}
    \end{table}

\end{itemize}

% summarize the results!!!!!!
The experimental results show: 
the unbiased risk estimator suffers severe overfitting and this issue is significantly alleviated by correction methods (U-correct, U-stop, U-flood and U-PRR) across various scenarios and datasets; 
the performance of directly applying flooding on the unbiased risk estimator can only reach the same level as early stopping; 
the partial risk regularization method outperformed all baseline methods and achieved the best performance for all the datasets and class prior settings.

\subsection{Robustness against Noisy Class Priors}
\label{sec:robustness}

\begin{figure*}[t]
    \centering
    \subfigure[Symmetric(a=0.5)]{\includegraphics[height=2.7cm,width=3.5cm]{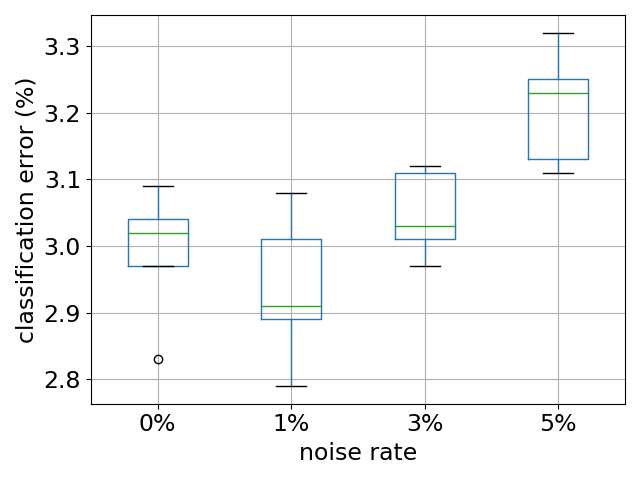}}
    \hfill
    \subfigure[Symmetric(a=0.1)]{\includegraphics[height=2.7cm,width=3.5cm]{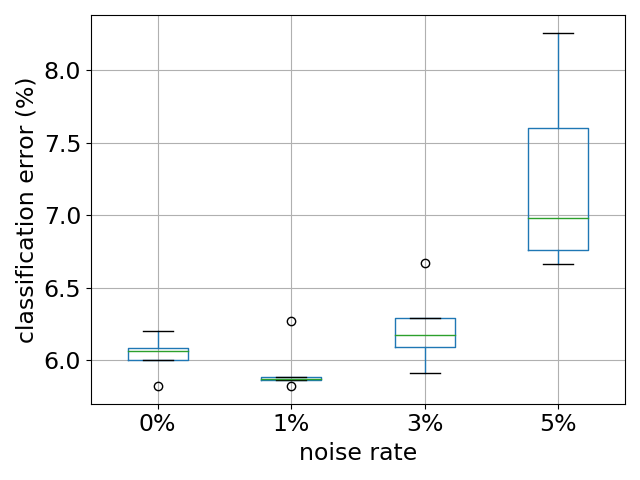}}
    \hfill
    \subfigure[Asymmetric]{\includegraphics[height=2.7cm,width=3.5cm]{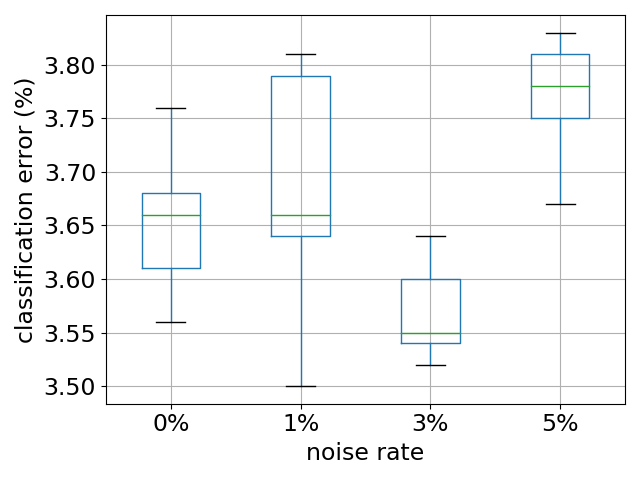}}
    \hfill
    \subfigure[Non-square]{\includegraphics[height=2.7cm,width=3.5cm]{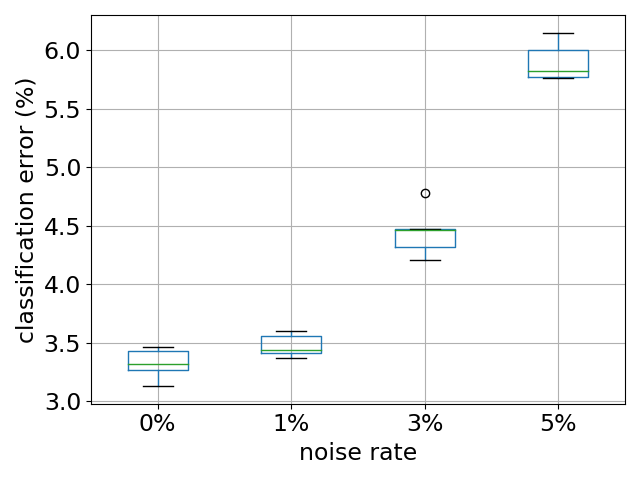}}
    \caption{Experiments of the robustness of the proposed method to different ($1\%$, $3\%$, and $5\%$) noise rate. The class-prior matrices in experiments are symmetric matrix with $a=0.5$ and $a=0.1$, asymmetric diagonal dominated matrix, and non-square matrix. Means (standard deviations) of the classification error over five trials in percentage.}
    \label{fig:Noisy}
\end{figure*}

Hitherto, we have assumed that the values of the class priors were accurately accessible, which may not be true in practice. We designed experiments where random noise was added to each class prior in order to simulate real-world situations. In the experiment, we added different levels ($1\%$, $3\%$, and $5\%$) of perturbation to each element of $\bm{\Theta}$ to obtain $\bm{\Theta^{'}}$, and the noisy $\bm{\Theta^{'}}$ was treated as the true $\bm{\Theta}$ during the whole learning process.
    
The results on the MNIST dataset with different class priors are shown in Fig. \ref{fig:Noisy}, where $0\%$ noise rate means that we used the true class priors. This figure shows that the proposed method is not strongly affected by noise in the class priors and thus it may be safely applied in the real world.

%% file: sections/conclusion.tex
\section{Conclusion}
\label{sec: concl}
In this work, we focused on the problem of multi-class classification from multiple unlabeled datasets. 
We established an unbiased risk estimator and provided a generalization error bound for it.
Based on our empirical observations, the negative empirical training risk was a potential reason why the unbiased risk estimation suffers severe overfitting.
To overcome overfitting, we proposed a partial risk regularization approach customized for learning from multiple unlabeled datasets.
Experiments demonstrated the superiority of our proposed partial risk regularization approach.

%% file: sections/appendix.tex
\appendix

\section{Proof of Theorem~\ref{thm: risk rewriting}}\label{apd:risk rewriting}
Notice that
\begin{align}
\label{eq: risk decomposition}
    R(\bg) = \sum_{k=1}^{K} \pi_k \mathbb{E}_{p_k} [\ell(\bg(\bx),k)] = \sum_{i=1}^{K} \sum_{j=1}^{K} (\Pi)_{i,j} \mathbb{E}_{p_i} [\ell(\bg(\bx),j)].
\end{align}
If the classification risk is rewritable, i.e., 
\begin{align*}
    R(\bg)
    & = \sum_{m=1}^{M} \bE_{q_m}\left[\sum_{k=1}^K w_{m, k} \ell(\bg(\bx), k)\right] = \sum_{m=1}^{M} \sum_{k'=1}^{K} \theta_{m,k'} \mathbb{E}_{p_{k'}}\left[\sum_{k=1}^K w_{m, k} \ell(\bg(\bx), k)\right] \\
    & = \sum_{i=1}^{K} \sum_{j=1}^{K} (W^{\top} \Theta)_{i,j} \mathbb{E}_{p_i} [\ell(\bg(\bx),j)],
\end{align*}
then comparing the coefficients with \eqref{eq: risk decomposition}, we have $W^{\top} \Theta =\Pi$.
Therefore, $R(\bg)$ is rewritable by letting $W = (\Pi\Theta^\dagger)^\top$.

\section{Proof of Theorem~\ref{thm: bound}}\label{apd:bound}
Our proof of the estimation error bound is based on \emph{Rademacher complexity} \citep{bartlett2002rademacher}. First, we show the uniform deviation bound, which is useful to derive the estimation error bound.

\begin{lemma}
\label{lemma: the uniform deviation bound}
For any $\delta > 0$, we have the probability at least $1 - \delta$,
\begin{align}
    \label{eq: the uniform deviation bound}
    \sup_{\bg \in \cG} \left| \widehat{R}_{\textup{U}}(\bg) - R(\bg)\right|
    \leq 2 \sqrt{2} LKC_{w} \sum_{m=1}^{M} \sum_{y=1}^{K} \mathfrak{R}_{n_m}(\cH_{y}) + 2C_{w}KC_{\ell} \sqrt{\frac{{}M\ln{\frac{2}{\delta}}}{2N_{min}}}.
\end{align}
where the probability is over repeated sampling of data for evaluating $\widehat{R}_{\text{U}}(\bg)$.
\end{lemma}

\begin{proof}
Consider the one-side uniform deviation $\sup_{\bg \in \cG} \widehat{R}_{\text{U}}(\bg) - R(\bg)$. The change of it will be no more than $\frac{C_w K C_\ell}{n_m}$ if some $\bx \in \cX_{m}$ ($1 \leq m \leq M$) is replaced by $\bx'$. Subsequently, \emph{McDiarmid’s inequality} \citep{mcdiarmid1989method} tells us that
\begin{align*}
    \Pr \left\{\sup_{\bg \in \cG} \widehat{R}_{\text{U}}(\bg) - R(\bg) - \bE [\sup_{\bg \in \cG} \widehat{R}_{\text{U}}(\bg) - R(\bg)] \geq \epsilon \right\} 
    \leq \exp(-\frac{2\epsilon^2}{(\sum_m \frac{1}{n_m})(C_w K C_\ell)^2}).
\end{align*}
or equivalently, with probability at least $1 - \delta/2$,
\begin{align*}
    \sup_{\bg \in \cG} \widehat{R}_{\text{U}}(\bg) - R(\bg)
    & \leq \bE [\sup_{\bg \in \cG} \widehat{R}_{\text{U}}(\bg) - R(\bg)]
    + \sqrt{\frac{(\sum_{m}\frac{1}{n_m})(C_w K C_\ell)^2 \ln\frac{2}{\delta}}{2}} \\
    & \leq \bE [\sup_{\bg \in \cG} \widehat{R}_{\text{U}}(\bg) - R(\bg)]
    + C_w K C_\ell \sqrt{\frac{M \ln{\frac{2}{\delta}}}{2 N_{\min}}}.
\end{align*}
By symmetrization \citep{Vapnik1998} and using the same trick in \citet{maurer2016vector}, it is a routine work to show that 
\begin{align*}
    \bE [\sup_{\bg \in \cG} \widehat{R}_{\text{U}}(\bg) - R(\bg)] 
    & \leq 2 \sqrt{2} L \sum_{m=1}^{M} \sum_{k=1}^{K} |w_{m,k}| \sum_{k'=1}^{K} \mathfrak{R}_{n_m}(\cH_{k'}) \\
    & \leq 2 \sqrt{2} LKC_{w} \sum_{m=1}^{M} \sum_{k=1}^{K} \mathfrak{R}_{n_m}(\cH_{k}).
\end{align*}
The one-side uniform deviation $R(\bg) - \sup_{\bg \in \cG} \widehat{R}_{\text{U}}(\bg)$ can be bounded similarly.
\end{proof}

Based on Lemma~\ref{lemma: the uniform deviation bound}, the estimation error bound (\ref{eq: bound of unbiased risk estimator}) is proven through 
\begin{align*}
    R(\widehat{\bg}_{\textup{U}}) - R(\bg^{*})
    & = \left( \widehat{R}_{\text{U}}(\widehat{\bg}_{\textup{U}}) - \widehat{R}_{\text{U}}(\bg^*) \right)
    + \left( R(\widehat{\bg}_{\textup{U}}) -  \widehat{R}_{\text{U}}(\widehat{\bg}_{\textup{U}}) \right)
    + \left( \widehat{R}_{\text{U}}(\bg^*) - R(\bg^*) \right)\\
    & \leq 0 + 2\sup_{\bg \in \cG}\left| \widehat{R}_{\text{U}}(\bg) - R(\bg) \right| \\
    & \leq 4 \sqrt{2} LKC_{w} \sum_{m=1}^{M} \sum_{k=1}^{K} \mathfrak{R}_{n_m}(\cH_{k})
    + 2C_{w}KC_{\ell} \sqrt{\frac{M\ln{\frac{2}{\delta}}}{2N_{\min}}}.
\end{align*}
where $\widehat{R}_{\text{U}}(\widehat{\bg}_{\textup{U}}) \leq \widehat{R}_{\text{U}}(\bg^*)$ by the definition of $\widehat{\bg}_{\textup{U}}$.

\section{Proof of Lemma~\ref{lemma: float condition}}
By the relationship of distributions $p_k$ and $q_m$, we have
\begin{align*}
    b_{m,k} 
    = \sum_{i=1}^{K} \theta_{m,i} R^{01}_{p_i,k}(\bg^\star) 
    = R^{01}_{\sum_{i=1}^{K}\theta_{m,i}p_i,k}(\bg^\star) 
    = R^{01}_{q_m,k}(\bg^\star).
\end{align*}
When $\bg^\star$ is the Bayes optimal classifier and all classes are separable, we have
\begin{align*}
    R^{01}_{p_i,k}(\bg^\star)
    = \begin{cases}
    0 & \text{if } i = k, \\ 
    1 & \text{otherwise}. \\
    \end{cases}
\end{align*}
Then
\begin{align*}
    b_{m,k} 
    = \sum_{i=1}^{K} \theta_{m,i} R^{01}_{p_i,k}(\bg^\star) = \sum_{i \neq k} \theta_{m,i} = 1 - \theta_{m,k}
\end{align*}
holds.